\documentclass[twoside,11pt]{article}
\usepackage{fullpage}
\usepackage{amsmath,amsfonts,amssymb, amsthm}
\usepackage{color}
\usepackage{url}
\usepackage[ruled]{algorithm2e}
\usepackage{units}
\usepackage{psfrag}

\newtheorem{theorem}{Theorem}

\newtheorem{lemma}{Lemma}

\newcommand{\E}{\mathbb{E}}

\newcommand{\ceil}[1]{\left\lceil#1\right\rceil}

\newcommand{\Zcal}{\mathcal{Z}}

\newcommand{\secref}[1]{Sec.~\ref{#1}}

\renewcommand{\eqref}[1]{Eq.~(\ref{#1})}
\newcommand{\lemref}[1]{Lemma~\ref{#1}}
\newcommand{\thmref}[1]{Thm.~\ref{#1}}

\newcommand{\algref}[1]{Algorithm~\ref{#1}}

\newcommand{\argmin}{\mathop{\mathrm{arg\,min}{}}}

\newcommand{\varp}{\sigma}
\newcommand{\smoothp}{L}
\newcommand{\convexset}{W}

\newcommand{\regretbound}{\psi}

\title{Robust Distributed Online Prediction}
\author{Ofer Dekel, Ran Gilad-Bachrach, Ohad Shamir and Lin Xiao \\ Microsoft Research \\ \texttt{\{oferd,rang,ohadsh,lin.xiao\}@microsoft.com}}

\date{}

\begin{document}
\maketitle

\abstract
The standard model of online prediction deals with serial processing of inputs by a single processor. However, in large-scale online prediction problems, where inputs arrive at a high rate, an increasingly common necessity is to distribute the computation across several processors. A non-trivial challenge is to design distributed algorithms for online prediction, which maintain good regret guarantees. In \cite{DMB}, we presented the DMB algorithm, which is a generic framework to convert any serial gradient-based online prediction algorithm into a distributed algorithm. Moreover, its regret guarantee is asymptotically optimal for smooth convex loss functions and stochastic inputs. On the flip side, it is fragile to many types of failures that are common in distributed environments. In this companion paper, we present variants of the DMB algorithm, which are resilient to many types of network failures, and tolerant to varying performance of the computing nodes.

\section{Introduction}
In online prediction problems, one needs to provide predictions over a stream of inputs, while attempting to learn from the data and improve the predictions. Unlike offline settings, where the learning phase over a training set is decoupled from the testing phase, here the two are intertwined, and we cannot afford to slow down.

The standard models of online prediction consider a serial setting, where the inputs arrive one by one, and are processed by a single processor. However, in large-scale applications, such as search engines and cloud computing, the rate at which inputs arrive may necessitate distributing the computation across multiple cores or cluster nodes. A non-trivial challenge is to design distributed algorithms for online prediction, which maintain regret guarantees as close as possible to the serial case (that is, the ideal case where we would have been able to process all inputs using a single, sufficiently fast processor).

In \cite{DMB}, we presented the DMB algorithm, which is a template that allows to convert any serial online learning algorithm into a distributed algorithm. For a wide class of such algorithms, \cite{DMB} showed that when the loss function is smooth and the inputs are stochastic, then the DMB algorithm is asymptotically optimal. Specifically, the regret guarantee of the DMB algorithm will be identical in its leading term to the regret guarantee of the serial algorithm, including the constants. Also, the algorithm can be easily adapted to stochastic optimization problems, with an asymptotically optimal speedup in the convergence rate, by using a distributed system as opposed to a single processor.

However, the DMB algorithm 
makes several assumption that may not be realistic in all distributed settings.
These assumptions are:
\begin{itemize}
\item All nodes work at the same rate.
\item All nodes are working properly throughout the execution of the algorithm.
\item The network connecting the nodes is stable during the execution of the algorithm.
\end{itemize}
These assumptions are not always realistic. Consider for example a multi-core CPU. While the last two assumptions are reasonable in this environment, the first one is invalid since other processes running on the same CPU may cause occasional delays on some cores (e.g., \cite{PKP03}). In massively distributed, geographically dispersed systems, all three assumptions may fail to hold.

In this companion paper to \cite{DMB}, we focus on adding robustness to the DMB algorithm, and present two methods to achieve this goal. In \secref{sec:leader} we present ways in which the DMB algorithm can be made robust using a master-workers architecture, and relying on the robustness of off-the-shelf methods such as leader election algorithms or databases.
In \secref{sec:async}, we present an asynchronous version of the DMB algorithm 
that is robust with a fully decentralized architecture.

\section{Background}

We begin by providing a brief background on the setting and the DMB algorithm. The background is deliberately terse, and we refer the reader to \cite{DMB} for the full details.

We assume that we observe a stream of inputs $z_{1},z_{2},\ldots$, where
each~$z_{i}$ is sampled independently from a fixed unknown
distribution over a sample space $\Zcal$. Before observing each
$z_{i}$, we predict a point $w_{i}$ from a convex set~$\convexset$.  After
making the prediction~$w_{i}$, we observe~$z_{i}$ and suffer the loss
$f(w_{i},z_{i})$, where $f$ is a predefined loss function, assumed to be convex in its first argument.  We may now
use~$z_{i}$ to improve our prediction mechanism for the future (e.g.,
using a stochastic gradient method).  The goal is to accumulate the
smallest possible loss as we process the sequence of inputs.  More
specifically, we measure the quality of our predictions on $m$ examples using the notion of \emph{regret}, defined as
\[
R(m) = \sum_{i=1}^m \left( f(w_i,z_i) - f(w^\star,z_i) \right),
\]
where $w^\star = \argmin_{w \in \convexset} \E_z[f(w,z)]$. 
Note that the regret $R(m)$is a random variable, since it depends on~$m$
stochastic inputs. 
For simplicity, we will focus on bounding the expected regret.

We model our distributed computing system as a set of nodes, each of which is an independent processor, and a network that enables the nodes to communicate with each other. Each node receives an incoming stream of examples from an outside source, such as a load balancer/splitter. As in the real world, we
assume that the network has a limited bandwidth, so the nodes cannot
simply share all of their information, and that messages sent over the
network incur a non-negligible latency. 

The ideal (but unrealistic) solution to this online prediction problem is to run a standard serial algorithm on a single ``super'' processor that is sufficiently fast to handle the stream of examples. This solution is optimal, simply because any distributed algorithm can be simulated on a fast-enough single processor. The optimal regret that can be achieved by such serial algorithms on $m$ inputs is $O(\sqrt{m})$. However, when we choose to distribute the computation, the regret performance might degrade, as the communication between the nodes is limited. Straightforward approaches, as well as previous approaches in the literature, all yield regret bounds which are at best $O(\sqrt{km})$, where $k$ is the number of nodes in the system. Thus, the regret degrades rapidly as more nodes are utilized. 

In \cite{DMB}, we present the DMB algorithm, which has the following two important properties:
\begin{itemize}
\item It can use a wide class of gradient-based update rule for serial online
  prediction as a black box, and convert it into a parallel or
  distributed online prediction algorithm.
These serial online algorithms include (Euclidean) gradient descent, 
mirror descent, and dual averaging.
\item 
If the loss function~$f(w,z)$ is smooth in~$w$ (namely, its gradient is Lipschitz), then the DMB algorithm attains an asymptotically optimal regret bound of $O(\sqrt{m})$.  Moreover, the
  coefficient of the dominant term~$\sqrt{m}$ is the same as in the
  serial bound, which is \emph{independent} of~$k$ and of the network
  topology.
\end{itemize}

The DMB algorithm is based on a theoretical observation that, for smooth loss
functions, one can prove regret bounds for serial gradient-based algorithms 
that depend on the variance of the stochastic gradients. 
To simplify discussions, we use $\psi(\sigma^2,m)$ to denote such variance 
bounds for predicting~$m$ inputs, where $\sigma^2$ satisfies
\[
\forall\, w \in \convexset, \qquad \E_z \left[ \bigl\| \nabla_w f(w,z)
- \nabla_w \E_z f(w,z) ] \bigr\|^2 \right] \leq \varp^2 ~.
\]
For example, we show in \cite{DMB} that for both mirror-descent (including 
classical gradient descent) and dual averaging methods, the expected regret
bounds take the form
\[
\psi(m,\sigma^2) = 2D^2\smoothp+2D\varp\sqrt{m},
\]
where $\smoothp$ is the Lipschitz parameter of the loss gradient $\nabla_w f(w,z)$, and $D$ quantifies the size of the set $\convexset$ from which the predictors are chosen. As a result, it can be shown that applying a serial gradient-based algorithm on \emph{averages} of gradients, computed on independent examples with the same predictor $w$, will reduce the variance in the resulting regret bound.

In a nutshell, the DMB algorithm uses the distributed network in order to rapidly accumulate gradients with respect to the same fixed predictor $w$. Once a mini-batch of sufficiently many gradients are accumulated (parameterized by $b$), the nodes collectively perform a vector-sum operation, which allows each node to obtain the average of these $b$ gradients. This average is then used to update their predictor, using some gradient-based online update rule as a black box. Note that the algorithm is inherently synchronous, as all nodes must use the same predictor and perform the averaging computations and updates at the same time. A detailed pseudo-code and additional details appear in \cite{DMB}.

The regret analysis for this algorithm is based on a parameter $\mu$, 
which bounds the number of inputs processed by the system during the 
vector-sum operation. 
The gradients for these~$\mu$ inputs are not used for updating the predictor.
While $\mu$ depends on the network structure and communication latencies, it does not scale with the total number of examples $m$ processed by the system. Formally, the regret guarantee is as follows:
\begin{theorem} \label{thm:synchronous}
Let $f$ be an $\smoothp$-smooth convex loss function and assume that
the stochastic gradient $\nabla_w f(w,z_i)$ has $\varp^2$-bounded variance
for all~$w\in\convexset$.
If the online update rule used by the DMB algorithm has the serial regret bound $\regretbound(\varp^2, m)$, then the expected regret of the DMB algorithm over $m$ examples is at most
\[
(b+\mu)\,\regretbound\left(\frac{\varp^2}{b},
\left\lceil\frac{m}{b+\mu}\right\rceil\right) ~.
\]
Specifically, if $\regretbound(\varp^2, m)=2D^2\smoothp+2D\varp\sqrt{m}$,
and the batch size is chosen to be $b=m^\rho$ for any $\rho\in (0,1/2)$, the expected regret is $2D\sigma\sqrt{m} + o(\sqrt{m})$.
\end{theorem}
Note that for serial regret bounds of the form $2D^2\smoothp+2D\varp\sqrt{m}$, we indeed get an identical leading term in the regret bound for the DMB algorithm, implying its asymptotic optimality.

\section{Robust Learning with a Master-Workers Architecture}\label{sec:leader}

The DMB algorithm presented in \cite{DMB} assumes that all nodes are making similar progress. However, even in homogeneous systems, which are
designed to support synchronous programs, this is hard to achieve
(e.g., \cite{PKP03}), let alone grid environments in which each
node may have different capabilities. In this section, we present a variant of the DMB algorithm that adds the following properties:
\begin{itemize}
\item It performs on heterogeneous clusters, whose nodes may have varying processing rates.
\item It can handle dynamic network latencies.
\item It supports randomized update rules.
\item It can be made robust using standard fault tolerance techniques.
\end{itemize}

To provide these properties, we convert the DMB algorithm to work with
a single master and multiple workers. Each of the workers receives
inputs and processes them at its own pace. Periodically, the worker
sends the information it collected, i.e., the sum of gradients, to
the master. Once the master has collected sufficiently many gradients,
it performs an update and broadcasts the new predictor to the
workers. We call this algorithm the \emph{master-worker distributed
  mini-batches} (MaWo-DMB) algorithm. For a detailed description of
the algorithm, see \algref{alg:MaWoW} for the worker algorithm and
\algref{alg:MaWoM} for the master algorithm.

\begin{algorithm}[t]
\DontPrintSemicolon
initialize $w$\;
$j = 1$\;
count = 0\;
$\hat g$ = 0\;
\While{not end of data}
{
  \If {master message $\mathrm{m} = (w, j)$ and $\mathrm{m}.j > j$}
    {
       $w~:=\mathrm{m}.w$\;
       $j~:=~\mathrm{m}.j$\;
       $\hat g~:=~0$\;
       count$~:=~0$\;
    }
  \If {did not send message for the past $t$ time--units and count $>0$}
    {
      send the message ($\hat g$, count, $j$) to the master\;
       $\hat g~:=~0$\;
       count$~:=~0$ \;
    }
     predict $w$\;
     receive input $z$ and suffer loss $f(w,z)$\;
     compute gradient $\nabla_w f(w,z)$\;
     $\hat g ~:=~\hat g + \nabla_w f(w,z)$\;
     count $~:=~$ count + 1\;
}

 \caption{MaWo-DMB worker algorithm.}
 \label{alg:MaWoW}
 \end{algorithm}

\begin{algorithm}[t]
\DontPrintSemicolon
$j = 1$\;
count = 0\;
$\hat g$ = 0\;
\While{not end of data}
{
  receive message m = ($\hat g$, count, $j$) from a worker\;
  \If {$\mathrm{m}.j = j$}
  {
     $\hat g ~:=~ \hat g + \mathrm{m.}\hat g$ \;
     count $~:=~$ count + m.count \;
     \If {count $\geq$ b}
    {
       $\bar{g}_{j} ~:=~ \frac {\hat g}{\text{count}}$ \;
       use $\bar{g}_{j}$ to compute updated predictor $w_{j+1}$\;
       $j~:=~j+1$ \;
       count $~:=~$ 0\;
       $\hat g ~:=~ 0$\;
       broadcast ($w_{j}$,$j$)\;
     }
  }
}

 \caption{MaWo-DMB master algorithm.}
 \label{alg:MaWoM}
 \end{algorithm}

This algorithm uses a slightly different communication protocol than the DMB algorithm. We assume that the network supports two operations:
\begin{enumerate}
\item Broadcast master $\rightarrow$ workers: the master sends updates to the workers.
\item Message worker $\rightarrow$ master: periodically, each worker sends a message to the master with the sum of gradients it has collected so far.
\end{enumerate}
One possible method to implement these services is via a database. Using a database, each worker can update the gradients it collected on the database,
and check for updates from the master. At the same time, the master can check periodically to see if sufficiently many gradients have accumulated in the
database. When there are at least $b$ gradients accumulated, the master performs an update and posts the result in a designated place in the database.
This method provides nice robustness features to the algorithm, as discussed in \secref{sec:MW robust}.

\subsection{Properties of the MaWo-DMB algorithm}
The MaWo-DMB algorithm shares a similar asymptotic behavior as the DMB
algorithm (e.g. as discussed in \thmref{thm:synchronous}). The proof for the DMB algorithm applies to this algorithm as well. To get the optimal rate, we only need to bound the number $\mu$ of inputs whose gradient is not used in the computation of the next prediction point.  A coarse bound on this number can be given as follows: Let $M$ be the maximal number of inputs per time--unit. Let $T$ be the time difference between messages sent from each worker to the
master, let $\tau_u$ be the time it takes the master to perform an
update, and let $\tau_c$ be the maximal time it takes to send a
message between two points in the network. Using this notation, the
number of inputs dropped in each update is at most
$M(T+2\tau_c+\tau_u)$. Specifically, let $t$ be the time when the
master encounters the $b$'th gradient. Inputs that were processed
before time $t - T - \tau_c$ were received by the master. Moreover, at
time $t + \tau_u + \tau_c$ all of the workers have already received
the updated prediction point. Therefore, only inputs that were
processed between $t - T - \tau_c$ and $t + \tau_u + \tau_c$ might be
dropped.  Clearly, there are at most $M(T + 2\tau_c + \tau_u)$ such
inputs.

While asymptotically the MaWo-DMB algorithm exhibits the same performance as the DMB algorithm, it does have some additional features. First, it allows
workers of different abilities to be used. Indeed, if some workers can process more inputs than other workers, the algorithm can compensate for
that. Moreover, the algorithm does not assume that the number of inputs each worker handles is fixed in time. Furthermore, workers can be added
and removed during the execution of the algorithm.

The DMB algorithm assumes that the update rule is deterministic. This is essential since each node computes the update, and it is assumed that
they reach the same result. However, in the MaWo-DMB algorithm, only the master calculates the update and sends it to the rest of the nodes,
therefore, the nodes all share the same point even if the update rule is randomized.


\subsection{Adding Fault Tolerance to the MaWo-DMB algorithm}\label{sec:MW robust}
The MaWo-DMB algorithm is not sensitive to the stability of the workers. Indeed, workers may be added and removed during the execution of the algorithm.
However, if the master fails, the algorithm stops making updates. This is a standard problem in master-worker environments. It can be solved
using leader election algorithms such as the algorithm of \cite{GHS83}. If the workers do not receive any signal from the master
for a long period of time, they start a process by
which they elect a new leader (master). \cite{MWW00} proposed a leader election algorithm for ad-hoc networks. The advantage
of this kind of algorithm for our setting is that it can manage dynamic networks where the network can be partitioned and reconnected. Therefore,
if the network becomes partitioned, each connected component will have its own master.

Another way to introduce robustness to the MaWo-DMB algorithm is by selecting the master only when an update step is to be made. Assume that there
is a central database and all workers update it. Every $T$ time--units, each worker performs the following
\begin{enumerate}
\item lock the record in the database
\item add the gradients computed to the sum of gradients reported in the database
\item add the number of gradients to the count of the gradients reported in the database
\end{enumerate}
At this point, the worker checks if the count of gradients exceeds
$b$. If it does not, the worker releases the lock and returns to
processing inputs. However, if the number of gradients does exceed
$b$, the worker performs the update and broadcasts the new prediction
point (using the database) before unlocking the database and becoming
a worker again.

This simple modification we just described creates a distributed master such that any node in the system can be removed without
significantly affecting the progress of the algorithm. In a sense, we are leveraging the reliability of the database system  (see e.g., \cite{BHG87, DeanBrock, bigtable})
to convert our algorithm into a fault tolerant algorithm.

\section{Robust Learning with a Decentralized Architecture}\label{sec:async}

In the previous section, we discussed asynchronous algorithms based on a master-workers paradigm. Using off-the-shelf fault tolerance methods, one can design simple and robust variants, capable of coping with dynamic and heterogeneous networks.

That being said, this kind of approach also has some
limitations. First of all, access to a shared database may not be feasible, particularly in massively distributed environments. Second, utilizing leader  election algorithms is potentially wasteful, since by the
time a new master is elected, some workers or local worker groups
might have already accumulated more than enough gradients to perform a
gradient update. Moreover, what we really need is in fact more complex
than just electing a random node as a master: electing a
computationally weak or communication-constrained node will have
severe repercussions. Also, unless the communication network is
fully connected, we will need to form an entire DAG (directed acyclic graph)
to relay gradients from the workers to the elected master. While both
issues have been studied in the literature, it complicates the
algorithms and increases the time required for the election process,
again leading to potential waste. In terms of performance guarantees,
it is hard to come up with explicit time guarantees for these
algorithms, and hence the effect on the regret incurred by the system
is unclear.

In this section, we describe a robust, fully decentralized and
asynchronous version of DMB, which is not based on a master-worker
paradigm. We call this algorithm \emph{asynchronous} DMB, or ADMB for
brevity. We provide a formal analysis, including an explicit regret
guarantee, and show that ADMB shares the advantages of DMB in
terms of dependence on network size and communication latency.

\subsection{Description of the ADMB Algorithm}

We assume that communication between nodes takes place along some bounded-degree acyclic graph. In addition, each node has a unique numerical index. We will generally use $i$ to denote a given node's index, and let $j$ denote the index of some neighboring node.

Informally, the algorithm works as follows: each node $i$ receives examples, accumulates gradients with respect to its current predictor (which we shall denote as $w_i$), and uses batches of $b$ such gradients to update the predictor. Note that unlike the MaWo-DMB algorithm, here there is no centralized master node responsible for performing the update. Also, for technical reasons, the prediction themselves are not made with the current predictor $w_i$, but rather with a running average $\bar{w}_i$ of predictors computed so far.

Each node occasionally sends its current predictor and accumulated gradients to its neighboring nodes. Given a message from a node $j$, the receiving node $i$ compares its state to the state of node $j$. If $w_i=w_j$, then both nodes have been accumulating gradients with respect to the same predictor. Thus, node $i$ can use these gradients to update its own predictor $w_i$, so it stores these gradients. Later on, these gradients are sent in turn to node $i$'s neighbors, and so on. Each node keeps track of which gradients came from which neighboring nodes, and ensures that no gradient is ever sent back to the node from which it came. This allows for the gradients to propagate throughout the network.

An additional twist is that in the ADMB algorithm, we no longer insist on all nodes sharing the exact same predictor at any given time point. Of course, this can lead to each node using a different predictor, so no node will be able to use the gradients of any other node, and the system will behave as if the nodes all run in isolation. To prevent this, we add a mechanism, which ensures
that if a node $i$ receives from a neighbor node $j$ a ``better'' predictor than its current one, it will switch to using node $j$'s predictor. By ``better'', we mean one of two things: either $w_j$ was obtained based on more predictor updates, or $j<i$. In the former case, $w_j,\bar{w}_j$ should indeed be better, since they are based on more updates. In the latter case, there is no real reason to prefer one or the other, but we use an order of precedence between the nodes to determine who should synchronize with whom. With this mechanism, the predictor with the most gradient updates is propagated quickly throughout the system, so either everyone starts working with this predictor and share gradients, or an even better predictor is obtained somewhere in the system, and is then quickly propagated in turn - a win-win situation.

We now turn to describe the algorithm formally. The algorithm has two  global parameters:
 \begin{itemize}
    \item $b$: As in the DMB algorithm, $b$ is the number of gradients whose average is used to update the predictor.
    \item $t$: This parameter regulates the communication rate between the nodes. Each node $i$ will send message to its neighbor every $t$ time--units.
 \end{itemize}

Each node $i$ maintains the following data structures:
 \begin{itemize}
     \item A \emph{node state} $S_i=(w_i,\bar{w}_i,v_i)$, where
      \begin{itemize}
          \item $w_i$ is the current predictor.
          \item $\bar{w}_i$ is the running average of predictors actually used for prediction.
          \item $v_i$ counts how many predictors are averaged in $\bar{w}_i$. This is also the number of updates performed according to the online update rule, in order to obtain $w_i$.
      \end{itemize}
      \item A vector $g_i$ and associated counter $c_i$, which hold the sum of gradients computed from inputs serviced by node $i$.
      \item For each neighboring node $j$, a vector $g_i^j$ and associated counter $c_i^j$, which hold the sum of gradients received from node $j$.
 \end{itemize}

When a node $i$ is initialized, all the variables discussed above are set to zero, The node then begins the execution of the algorithm. The protocol is composed of executing three event-driven functions: the first function (Algorithm \ref{alg:asyncfunc} below) is executed when a new request for prediction arrives, and handles the processing of that example. The second function (Algorithm \ref{alg:asyncsend}) is executed every $t$ time--units, and sends messages to the node's neighbors. The third function (Algorithm \ref{alg:asyncreceive}) is executed when a message arrives from a neighboring node. Also, the functions use a subroutine \texttt{update\_predictor} (Algorithm \ref{alg:updatepredictor}) to update the node's predictor if needed. For simplicity, we will assume that each of those three functions is executed atomically (namely, only one of the function runs at any given time). While this assumption can be easily relaxed, it allows us to avoid a tedious discussion of shared resource synchronization between the functions.

\begin{algorithm}
\DontPrintSemicolon
Predict using $\bar{w}_i$\;
Receive input $z$, suffer loss and compute gradient $\nabla_{w} f(w_i,z)$\;
$g_i:=g_i+\nabla_{w} f(w_i,z)$~,~$c_i:=c_i+1$\;
\If{$c_i+\sum_j c_i^j \geq b$}{\texttt{update\_predictor\;}}
\caption{ADMB Algorithm: Handling a new request} \label{alg:asyncfunc}
\end{algorithm}

\begin{algorithm}
\DontPrintSemicolon
For each neighboring node $j'$, send message $\left(i,S_i,g_i+\sum_{j\neq j'}g_i^j,c_i+\sum_{j\neq j'}c_i^j\right)$
\caption{ADMB Algorithm: Sending Messages (Every $t$ Time--Units) } \label{alg:asyncsend}
\end{algorithm}

\begin{algorithm}
 \DontPrintSemicolon
Let $(j,S_j,g,c)$ be the received message\;
\eIf{$S_j.v_j>v_i$ or ($S_j.v_j=v_i$ and $S_j.w_j\neq w_i$ and $j<i$)}
    {
    $S_i:=S_j$~,~$g_i:=0$~,~$c_i:=0$\;
    $\forall j$~~$g_i^j:=g$~,~$c_i^j:=c$\;
    }
{
\If{$S_j.w_j=w_i$}
{
    $g_i^j=g$~,~$c_i^j=c$\;
    \If{$c_i+\sum_j c_i^j \geq b$}
        {\texttt{update\_predictor}\;}
}
}
\caption{ADMB Algorithm: Processing Incoming Message} \label{alg:asyncreceive}
\end{algorithm}

\begin{algorithm}
\DontPrintSemicolon
use averaged gradient $\frac{g_i+\sum_j g_i^j}{c_i+\sum_j c_i^j}$ to compute updated predictor $w_{i}$\;
$\bar{w}_i ~:=~ \frac{v_i}{v_i+1} \bar{w}_i+\frac{1}{v_i+1}w_i$\;
$v_i:= v_i+1$~,~$g_i:= 0$~,~$c_i:= 0$\;
$\forall j$~~$g_i^j:= 0$~,~$c_i^j:= 0$\;
\caption{\texttt{update\_predictor} Subroutine} \label{alg:updatepredictor}
\end{algorithm}

It is not hard to verify that due to the acyclic structure of the network, no single gradient is ever propagated to the same node twice. Thus, the algorithm indeed works correctly, in the sense that the updates are always performed based on independent gradients. Moreover, the algorithm is well-behaved in terms of traffic volume over the network, since any communication link from node $i$ to node $j$ passes at most $1$ message every $t$ time--units, where  $t$ is a tunable parameter.

As with the MaWo-DMB algorithm, the ADMB algorithm has some desirable robustness properties, such as heterogeneous nodes and adding/removing new nodes, and communication latencies. Moreover, it is robust to network failures: even if the the network is split into two (or more) partitions, it only means we end up with two (or more) networks which implement the algorithm in isolation. The system can continue to run and its output will remain valid, although the predictor update rate will become somewhat slower, until the failed node is replaced. Note that unlike the MaWo-DMB algorithm, there is no need to wait until a master node is elected.

\subsection{Analysis}\label{subsec:analysis}

We now turn to discuss the regret performance of the algorithm. Before we begin, it is important to understand what kind of guarantees are possible in such a setting. In particular, it is not possible to provide a total regret bound over all the examples fed to the system, since we have not specified what happens to the examples which were sent to malfunctioning nodes - whether they were dropped, rerouted to a different node and so on. Moreover, even if nodes behave properly in terms of processing incoming examples, the performance of components such as interaction with neighboring nodes might vary over time in complex ways, which are hard to model precisely.

Instead, we will isolate a set of ``well-behaved'' nodes, and focus on the regret incurred on the examples sent to these nodes. The underlying assumption is that the system is mostly functional for most of the time, so the large majority of examples are processed by such well-behaved nodes. The analysis will focus on obtaining regret bounds over these examples.

To that end, let us focus on a particular set of $k'$ nodes, which form a connected component of the communication framework, with diameter $d'$. We will define the nodes as \emph{good}, if all those nodes implement the ADMB algorithm at a reasonably fast rate. More precisely, we will require the following from each of the $k'$ nodes:
\begin{itemize}
\item Executing each of the three functions defining the ADMB algorithm takes at most one time--unit.
\item The communication latency between two adjacent nodes is at most one time--unit.
\item The $k'$ nodes receive at most $M$ examples every time--unit.
\end{itemize}
As to other nodes, we only assume that the messages they send to the good
nodes reflect a correct node state, as specified earlier. In particular, they may be arbitrarily slow or even completely unresponsive.

First, we show that when the nodes are good, up-to-date predictors from any single node will be rapidly propagated to all the other nodes. This shows that the system has good recovery properties (e.g. after most nodes fail).

\begin{lemma}\label{lem:predprop}
Assume that at some time point, the $k'$ nodes are good, and at least one of them has a predictor based on at least $v$ updates. If the nodes remain good for at least $(t+2)d'$ time--units, then all nodes will have a predictor based on at least $v$ updates.
\end{lemma}
\begin{proof}
Let $i$ be the node with the predictor having at least $v$ updates. Counting from the time point defined in the lemma, at most $t+2$ time--units will elapse until all of node $i$'s neighbors will receive a message from node $i$ with its predictor, and either switch to this predictor (and then will have a predictor with $v$ updates), or remain with the same predictor (and this can only happen if its predictor was already based on $v$ updates). In any case, during the next $t+2$ time--units, each of those neighboring nodes will send a message to its own neighbors, and so on. Since the distance between any two nodes is at most $d'$, the result follows.
\end{proof}

The next result shows that when all nodes are good and have a predictor based on at least $v$ updates, not too much time will pass until they will all update their predictor.

\begin{theorem}\label{thm:fastupdates}
Assume that at some time point, the $k'$ nodes are good, and every one of them has a predictor with $\geq v$ updates (not necessarily the same one). Then after the nodes process at most
\[
b+2(t+2)d'M
\]
additional examples, all $k'$ nodes will have a predictor based on at least $v+1$ updates.
\end{theorem}

\begin{proof}
Consider the time point mentioned in the theorem, where every one of the $k'$ nodes, and in particular the node $i_0$ with smallest index among them, has a predictor with $\geq v$ updates. We now claim that after processing at most
\begin{equation}\label{eq:timespan2}
(t+2)d' M
\end{equation}
examples, either some node in our set had a predictor with $\geq v+1$ updates, or every node has the same predictor based on $v$ updates. The argument is similar to \lemref{lem:predprop}, since everyone will switch to the predictor propagated from node $i_0$, assuming no predictor obtained a predictor with more updates. Therefore, at most $(t+2)d'$ time--units will pass, during which at most $(t+2)d' M$ examples are processed.

So suppose we are now at the time point, where either some node had a predictor with $\geq v+1$ updates, or every node had the same predictor based on $v$ updates. We now claim that after processing at most
\begin{equation}\label{eq:timespan3}
b+(t+2)d'M
\end{equation}
examples, any node in our set obtained a predictor with $\geq v+1$ updates. To justify \eqref{eq:timespan3}, let us consider first the case where every node had the same predictor based on $v$ updates. As shown above, the number of time--units it takes any single gradient to propagate to all $k'$ nodes is at most $(t+2)d'$. Therefore, after $T$ time--units elapsed, each node will accumulate and act upon all the gradients computed by all nodes up to time $T-(t+2)d'$. Since at most $M$ examples are processed each time--unit, it follows that after processing at most $b+(t+2)d'M$ examples, all nodes will update their predictors, as stated in \eqref{eq:timespan3}.

We still need to consider the second case, namely that some good node had a predictor with $\geq v+1$ updates, and we want to bound the number of examples processed till all nodes have a predictor with $\geq v+1$ updates. But this was already calculated to be at most $(t+2)d'M$, which is smaller than  \eqref{eq:timespan3}. Thus, the time bound in \eqref{eq:timespan3} covers this case as well.

Adding \eqref{eq:timespan2} and \eqref{eq:timespan3}, the theorem follows.
\end{proof}

With these results in hand, we can now prove a regret bound for our algorithm. To do so, define a \emph{good time period} to be a time during which:
\begin{itemize}
    \item All $k'$ nodes are good, and were also good for $(t+2)d'$ time--units prior to that time period.
    \item The $k'$ nodes handled $b+2(t+2)d'M$ examples overall.
\end{itemize}
As to other time periods, we will only assume that at least \emph{one} of the $k'$ nodes remained operational and implemented the ADMB algorithm (at an arbitrarily slow rate).

\begin{theorem}\label{thm:asyncregret}
Suppose the gradient-based update rule has the serial regret bound $\regretbound(\varp^2, m)$, and that for any $\varp^2$, $\frac{1}{m}\regretbound(\varp^2,m)$ decreases monotonically in $m$.

Let $m$ be the number of examples handled during a sequence of non-overlapping good time periods. Then the expected regret with respect to these examples is at most
\[
\sum_{j=1}^{\ceil{m/\mu}}\frac{\mu}{j}\regretbound\left(\frac{\varp^2}{b},j\right),
\]
where $\mu=b+2(t+2)d'M$.
Specifically, if $\regretbound(\varp^2, m)=2D^2\smoothp+2D\varp\sqrt{m}$, then the expected regret bound is
\[
2D^2L(b+2(t+2)d'M)(1+\log(m))+
4D\sigma\sqrt{\left(1+\frac{2(t+2)d'M}{b}\right)m}
\]
\end{theorem}

When the batch size $b$ scales as $m^{\rho}$ for any $\rho\in (0,1/2)$, we get an asymptotic regret bound of the form $4D\sigma\sqrt{m}+o(\sqrt{m})$. The leading term is virtually the same as the leading term in the serial regret bound. The only difference is an additional factor of $2$, essentially due to the fact that we need to average the predictors obtained so far to make the analysis go through, rather than just using the last predictor.

\begin{proof}
Let us number the good time periods as $j=1,2,\ldots$, and let $\bar{w}_j$ be a predictor used by one of the nodes at the beginning of the $j$-th good time period. From \lemref{lem:predprop} and \thmref{thm:fastupdates}, we know that the predictors used by the nodes were updated at least once during each period. Thus, $\bar{w}_j$ is the average of $j'\geq j$ predictors $w_1,w_2,\ldots,w_{j'}$, where each $w_{p+1}$ was obtained from the previous $w_{p}$ using $b_p\geq b$ gradients each, on some examples which we shall denote as $z_{p,1},z_{p,2},\ldots,z_{p,b_p}$. Since $w_p$ is independent of these examples, we get
\[
\E\left[\frac{1}{b_p}\sum_{q=1}^{b_p}f(w_p,z_{p,q})-f(w^\star,z_{p,q})~\big| w_p\right] = \E[f(w_p,z)-f(w^\star,z)\big|w_p].
\]
Based on this observation and Jensen's inequality, we have
\begin{align}
&\E\left[f(\bar{w}_j,z)-f(w^\star,z)\right]\notag\\
&\leq \frac{1}{j'}\E\left[\sum_{p=1}^{j'}f(w_p,z)-f(w^\star,z)\right]\notag\\
&= \frac{1}{j'}\E\left[\sum_{p=1}^{j'}\frac{1}{b_p}
\sum_{q=1}^{b_p}f(w_p,z_{p,q})-f(w^\star,z_{p,q})\right].\label{eq:jen}
\end{align}
The online update rule was performed on the averaged gradients obtained from $z_{p,1},\ldots,z_{p,b_p}$. This average gradient is equal to the gradient of the function $\frac{1}{b_p}\sum_{q=1}^{b_p}f(w_p,z_{p,q})$. Moreover, the variance of this gradient is at most $\varp^2/b_p\leq \varp^2/b$. Using the regret guarantee, we can upper bound \eqref{eq:jen} by
\[
\frac{1}{j'}\regretbound\left(\frac{\varp^2}{b},j'\right).
\]
Since $j'\geq j$, and since we assumed in the theorem statement that the expression above is monotonically decreasing in $j'$, we can upper bound it by
\[
\frac{1}{j}\regretbound\left(\frac{\varp^2}{b},j\right).
\]
From this sequence of inequalities, we get that for \emph{any} example processed by one of the $k'$ nodes during the good time period $j$, it holds that
\begin{equation}\label{eq:epochregret}
\E\left[f(\bar{w}_j,z)-f(w^\star,z)\right]\leq \frac{1}{j}\regretbound\left(\frac{\varp^2}{b},j\right).
\end{equation}
Let $\mu=b+2(t+2)d'M$ be the number of examples processed during each good time period. Since $m$ examples are processed overall, the total regret over all these examples is at most
\begin{equation}\label{eq:regretfinal}
\sum_{j=1}^{\ceil{m/\mu}}\frac{\mu}{j}\regretbound\left(\frac{\varp^2}{b},j\right).
\end{equation}

To get the specific regret form when $\regretbound(\varp^2,m)=2D^2L+2D\sigma\sqrt{m}$, we substitute into \eqref{eq:regretfinal}, and substitute $\mu=b+2(t+2)d'M$ to get
\begin{align*}
&\sum_{j=1}^{\ceil{m/\mu}}\left(2D^2L\frac{\mu}{j}+\frac{2D\sigma\mu}{\sqrt{b}}
\frac{1}{\sqrt{j}}\right)\\
&\leq 2D^2L\mu(1+\log(m))+
\frac{4D\sigma\mu}{\sqrt{b}}\sqrt{\frac{m}{\mu}}\\
&= 2D^2L(b+2(t+2)d'M)(1+\log(m))+
4D\sigma\sqrt{\left(1+\frac{2(t+2)d'M}{b}\right)m}.
\end{align*}
\end{proof}

\bibliographystyle{plain}
\bibliography{mybib}
\end{document}